\documentclass[11pt]{article} 
\usepackage{graphicx}

\usepackage[margin=1in]{geometry}
\usepackage{amsmath}
\usepackage{amsfonts}
\usepackage{dsfont}
\usepackage{amssymb}
\usepackage{amsthm}
\usepackage{algpseudocode}
\usepackage{algorithm}

\usepackage{xcolor}
\newtheorem{theorem}{Theorem}

\newtheorem{lemma}{Lemma}

\definecolor{kxcolor}{HTML}{1d3557}

\DeclareMathOperator{\R}{\mathbb{R}} 

\DeclareMathOperator{\St}{\mathcal{S}}

\newtheorem*{theorem*}{Theorem} 
\newtheorem*{corollary*}{Corollary}
\newtheorem*{lemma*}{Lemma}
\newtheorem*{proposition*}{Proposition}

\title{Policy Gradient for Reinforcement Learning with General Utilities}

\author{
Navdeep Kumar, Technion\\
\texttt{navdeepkumar@campus.technion.ac.il} \\
Kaixin Wang, NUS\\
\texttt{kaixin.wang@u.nus.edu} \\
Kfir Levy, Technion\\
\texttt{kfirylevy@technion.ac.il} \\
Shie Mannor, Technion\\
\texttt{shie@ee.technion.ac.il} \\
}

\begin{document}

\maketitle

\begin{abstract}
In Reinforcement Learning (RL), the goal of agents is to discover an optimal policy that maximizes the expected cumulative rewards. This objective may also be viewed as finding a policy that optimizes a linear function of its state-action occupancy measure, hereafter referred as Linear RL. However, many supervised and unsupervised RL problems are not covered in the Linear RL framework, such as apprenticeship learning, pure exploration and variational intrinsic control, where the objectives are non-linear functions of the occupancy measures.
RL with non-linear utilities looks unwieldy, as methods like Bellman equation, value iteration, policy gradient, dynamic programming that had tremendous success in Linear RL, fail to trivially generalize.
In this paper, we derive the policy gradient theorem for RL with general utilities. The policy gradient theorem proves to be a cornerstone in Linear RL due to its elegance and ease of implementability. Our policy gradient theorem for RL with general utilities shares the same elegance and ease of implementability.
Based on the policy gradient theorem derived, we also present a simple sample-based algorithm.
We believe our results will be of interest to the community and offer inspiration to future works in this generalized setting.
\end{abstract}
 



\section{Introduction}
Reinforcement Learning (RL) is a sequential decision problem where an agent interacts with an environment and learns to behave optimally. Agent observes the state of the environment and decides to take some action, then the environment evolves to a different state depending on the current state and current action of the agent, this process continues till termination or forever. In RL (henceforth Linear RL), the agents also receive a reward signal after taking an action that depends on both the state of the environment and the current action of the agent. The goal of the agent (in Linear RL) is to learn to take actions such that it maximizes its cumulative reward  \cite{Sutton1998}. It can also be seen as maximizing the objective that is linear in occupancy measure \cite{Puterman1994MarkovDP}. A significant body of work is dedicated to solving the RL problem efficiently in challenging domains \cite{mnih2015humanlevel, Dsilver2020}.\\

There are many ways of solving Linear RL using Bellman's equation, the value function, policy search, and dynamic programming. Among them, policy gradient methods in Linear RL are one of the most widely used methods that arise from the very elegant and easily implementable policy gradient theorem by \cite{PolicyGradient}. And recent works, theoretically guarantee the convergence to global optima in Linear RL \cite{agarwal2020theory}. Junyu et al. \cite{zhang2020variational} first proved that the exact gradient descent in policy space will converge to global minima in convex RL under some mild conditions.\\

However, not all sequential decision problems of interest take this form, for example, Apprentice Learning, diverse skill discovery, pure exploration and constrained MDPs, among others; see \cite{zahavy2021reward}. In general, an agent can be tasked to learn the policy that is a general function of occupation measure of the policies, henceforth General RL. And as a special case, when the objective function is convex (or concave) in the occupancy measure, it is known as Convex RL\cite{zhang2020variational,geist2022concave,zhang2020variational,zahavy2021reward,mutti2022challenging}.

\section{Background and Notations}

\textit{Notations} \quad In what follows, we use $\mathds{1}$ to denote the indicator function. $\langle\cdot\rangle$ denotes dot product. $\Delta_\mathcal{U}$ denotes the space of probability distributions over a set $\mathcal{U}$.

A Markov Decision Process (MDP) is defined by a sextuple $(\mathcal{S},\mathcal{A},\gamma, R,P, q)$, where $\mathcal{S}$ is the state space, $\mathcal{A}$ is the action space, $\gamma \in [0,1)$ is the discount factor, $R \in \mathbb{R}^{\mathcal{S}\times\mathcal{A}}$ is the reward vector, $P \in (\Delta_{\mathcal{S}})^{\mathcal{S}\times\mathcal{A}}$ is the transition kernel and $q \in \Delta_{\mathcal{S}}$ is the initial distribution over the state space $\mathcal{S}$~\cite{Sutton1998}. A stationary policy $\pi \in (\Delta_{\mathcal{A}})^\mathcal{S}$ maps states to probability distributions over actions. Moreover, $P(s'|s,a), \pi(a|s)$ represent the probability of transition from state $s$ under action $a$ to state $s'$ and probability of taking action $a$ in state $s$ by policy $\pi$ respectively. Let $\Pi$ be the set of all stationary policies.
Let $\mu^{\pi} \in (\Delta_{\mathcal{S}})^{\mathcal{S}\times\mathcal{A}}$ be the occupancy measure of policy $\pi$, defined as~\cite{Puterman1994MarkovDP}
\begin{equation}
    \mu^{\pi}(s,a):= \sum_{n=0}^{\infty}\gamma^n\mathbb{E}[\mathds{1}(s_n=s,a_n=a)|s_0\sim q,a_t\sim \pi(\cdot|s_t), s_{t}\sim P(\cdot|s_{t-1},a_{t-1}), \forall\,t, 1\le t\leq n].
\end{equation}
With slight abuse of notation, we denote $\sum_{a\in\mathcal{A}}\mu^{\pi}(s,a)$ with $\mu^{\pi}(s)$. In Linear RL problems, we maximize the expected discounted reward under the policy $\pi$, that is,
\begin{equation}
     \max_{\mu^{\pi}\in\mathcal{K}} \langle \mu^{\pi}, R\rangle,
\end{equation}
where $\mathcal{K}:=\{\mu^{\pi} |\pi\in\Pi\}$ is set of occupancy measure of all stationary policies \cite{Puterman1994MarkovDP}. Here, we consider a more general objective
\begin{equation}
    \min_{\mu^{\pi}\in\mathcal{K}} f(\mu^{\pi}) 
\end{equation}
where $f:\mathcal{K}\to\mathbb{R}$ is a differentiable function, which we refer to as General RL (RL with general utilities). The differentiablity of the function is assumed for the simplicity of presentation. In case of apprentice learning \cite{ApprenticeshipLearning} and pure exploration \cite{pureExploration}, objective functions are $f(\mu^\pi) = ||\mu^\pi -\mu_\text{expert}||^2$ and  $f(\mu^\pi) = \mu^\pi\cdot\log(\mu^\pi)$ respectively.
Linear RL is extensively studied \cite{Sutton1998}, and there have been a few works on Convex RL (RL with convex utilities)\cite{zhang2020variational,geist2022concave,zhang2020variational,zahavy2021reward,mutti2022challenging}.

\section{Policy Gradient for Convex MDPs}

To handle large state space, we often parameterize the policy with $\theta\in\Theta\subset\mathbb{R}^d$, \textit{i.e.}, $\pi=\pi_\theta$.
For clarity, we will slightly abuse the notation and simply write $\pi$.
To apply gradient descent on policy space (or parameter space), we need the gradient w.r.t policy (or policy parameter):
\begin{equation}
    \nabla_\theta f : = \frac{\mathrm{d}f(\mu^{\pi})}{\mathrm{d}\theta}. 
\end{equation}
By the chain rule, we have
\begin{equation}
    \nabla_\theta f(\mu^{\pi}) = (\nabla_{\mu^\pi}f(\mu^{\pi}))^\top\frac{\mathrm{d}\mu^{\pi}}{\mathrm{d}\theta}= \sum_{s\in\mathcal{S}}\sum_{a\in\mathcal{A}}\frac{\partial f(\mu^{\pi})}{\partial\mu^{\pi}(s,a)} \nabla_{\theta} \mu^\pi(s,a).
\end{equation}
Let $\beta^\pi_k(s,a,s',a')=\mathbb{E}[\mathds{1}(s_k=s',a_k=a')\mid s_0=s,a_0=a,a_t\sim \pi(\cdot|s_t), s_{t}\sim P(\cdot|s_{t-1},a_{t-1}), \forall t\leq k] $ is the probability of transition from  state-action $(s, a)$ to state-action $(s', a')$ in exactly $k$ steps following the policy $\pi$. In addition, we denote $\beta^\pi(s',a',s,a) = \sum_{k=0}^{\infty}\gamma^k\beta^\pi_k(s',a',s,a) $. To derive the policy gradient theorem for convex RL, let us first look at the gradient of the occupancy measure w.r.t. the policy parameters. We can obtain the following equation,
\begin{equation}\label{eqn:gradient-occupancy}
    \nabla_{\theta}\mu^\pi(s,a) = \sum_{(s',a')\in\mathcal{S}\times\mathcal{A}} \mu^{\pi}(s') \beta^\pi(s',a',s,a) \nabla_{\theta}\pi(a'|s'),
\end{equation}
with proof referred in the appendix. Let $Q^\pi_R$ be the Q-value function for the policy $\pi$ and reward vector $R$, that is defined as
\begin{align}
    Q^\pi_R(s,a) &:= \sum_{n=0}^\infty \gamma^n \mathbb{E}[R(s_n,a_n)|s_0=s,a_0=a,a_t\sim\pi(\cdot|s_t),s_{t}\sim P(\cdot|s_{t-1},a_{t-1}),\forall t\leq n]\\
    &=\sum_{s',a'}\beta^\pi(s,a,s',a')R(s',a').
\end{align}
Now we are ready to present the policy gradient theorem for convex RL.
\begin{theorem}\label{rs:pgt}(Policy Gradient Theorem for convex RL)
\begin{equation}
    \nabla_\theta f = \sum_{s,a}\mu^{\pi}(s)Q^{\pi}_{R_\pi}(s,a)\frac{\mathrm{d}\pi(s,a)}{\mathrm{d}\theta},\quad \text{where}\,  R_\pi =\frac{d f(x)}{d x}|_{x = \mu^\pi}.
\end{equation}
\end{theorem}
\begin{proof}
Using Eqn.~\ref{eqn:gradient-occupancy}, we have
\begin{equation}
\begin{aligned}
    \nabla_\theta f
    =& \sum_{s\in\mathcal{S}}\sum_{a\in\mathcal{A}}\frac{\partial f(\mu^{\pi})}{\partial\mu^{\pi}(s,a)} \nabla_{\theta} \mu^\pi(s,a) =  \sum_{s\in\mathcal{S}}\sum_{a\in\mathcal{A}}R_{\pi}(s,a) \sum_{(s',a')\in\mathcal{S}\times\mathcal{A}} \mu^{\pi}(s') \beta^\pi(s',a',s,a) \nabla_{\theta}\pi(a'|s') \\
    =& \sum_{(s',a')\in\mathcal{S}\times\mathcal{A}} \mu^{\pi}(s') \sum_{s\in\mathcal{S}}\sum_{a\in\mathcal{A}}R_{\pi}(s,a)  \beta^\pi(s',a',s,a) \nabla_{\theta}\pi(a'|s')=\sum_{s\in\mathcal{S}}\sum_{a\in\mathcal{A}}\mu^{\pi}(s)Q^{\pi}_{R_\pi}(s,a)\frac{\mathrm{d}\pi(s,a)}{\mathrm{d}\theta},
\end{aligned}
\end{equation}
which concludes the proof.
\end{proof}
We believe that this policy gradient can be implemented with the same ease as the policy gradient for Linear RL. Why? In the standard policy gradient, we estimate Q-value for each policy (keeping the reward fixed) while updating the policy. To estimate the Q-value, we run a few trajectories of the current policy. Now, policy gradient for convex RL can be implemented with the same trajectories except with an extra step as follows. We run a few trajectories of the current policies as before, then we estimate the occupation measure  to obtain the current reward vector, and finally, estimate the Q-value using this reward vector and the same trajectories. Note that, we used the same trajectories to estimate the reward vector and Q-value. Hence, we don't need anything extra. This process can be easily done in parallel too if the update step sizes are small enough. We formalize this intuition and propose an algorithm for RL with general utilities in the next section. 

\section{Policy Gradient with Approximations}
The proof here carries over exactly the same way as in section 2 in \cite{Sutton1998}. We present here, just for the sake of completeness. Let $g_w:\mathcal{S}\times\mathcal{A} \to \mathcal{R}$ be our approximation for $Q^\pi$, with parameter $w$. We use $Q^\pi$ as a shorthand for $Q^\pi_{R_\pi}$. Assuming local convergence,
\begin{align}\label{eq:local_convergence}
    \sum_{s,a}\mu^\pi(s)\pi(s,a)\big[Q^\pi(s,a)-g_w(s,a)\big]\frac{\partial g_w(s,a)}{\partial w} = 0,
\end{align}
and compatibility with the policy parametrization 
\begin{align}\label{eq:compatibility}
    \frac{\partial g_w(s,a)}{\partial w} = \frac{\partial \pi(s,a)}{\partial \theta} \frac{1}{\pi(s,a)}.
\end{align}
Combining \eqref{eq:local_convergence} and \eqref{eq:compatibility}, we have
\begin{align}
    \sum_{s,a}\mu^\pi(s)\frac{\partial \pi(s,a)}{\partial \theta}\big[Q^\pi(s,a)-g_w(s,a)\big] = 0.
\end{align}
Now subtracting this into the policy gradient theorem \ref{rs:pgt}, we have
\begin{align}
    \frac{\partial f}{\partial \theta} &= \sum_{s,a}\mu^{\pi}(s)Q^{\pi}(s,a)\frac{\mathrm{d}\pi(s,a)}{\mathrm{d}\theta}-\sum_{s,a}\mu^\pi(s)\frac{\partial \pi(s,a)}{\partial \theta}\big[Q^\pi(s,a)-g_w(s,a)\big]\\
    &= \sum_{s,a}\mu^\pi(s)\frac{\partial \pi(s,a)}{\partial \theta}\big[Q^\pi(s,a) -Q^\pi(s,a)+g_w(s,a)\big]\\
    &= \sum_{s,a}\mu^\pi(s)\frac{\partial \pi(s,a)}{\partial \theta}g_w(s,a).\\
\end{align}

\begin{theorem}(Policy Gradient with Function Approximation) If $g_w$ satisfies \eqref{eq:local_convergence} and \eqref{eq:compatibility}, then
    \begin{align*}
    \frac{\partial f}{\partial \theta}     &= \sum_{s,a}\mu^\pi(s)\frac{\partial \pi(s,a)}{\partial \theta}g_w(s,a).
\end{align*}
\end{theorem}
\begin{proof}
    Proved above.
\end{proof}

\section{Bootstrapping for Occupation Measure}
One way to estimate occupation measure is by counting frequency of different state visitation. But this approach in not scalable to large state space and continuous state space. Here, we use bootstraping method to estimate occupation measure that can be combined with neural network to handle large and continuous state space.
Recall that the occupation is defined as 
\begin{align}
    d^\pi_P &= \sum_{n=0}^\infty \gamma^n (P^\pi)^n \mu\\
    \implies  \gamma P^\pi d^\pi_P &= \sum_{n=1}^\infty \gamma^n (P^\pi)^n \mu\\
    &= d^\pi_P -\mu.\\
\end{align}
That is, we have
\[ d^\pi(s) = \mu(s)  + \gamma \sum_{s',a}\pi(a|s)P(s'|s,a)d^\pi_P(s').\]
Under appropriate step size sequence $\{\eta_t\}$, the update rule
\begin{align}
    d(s_t) = d(s_t) + \eta_t\Bigm[\mu(s_t) + \gamma d(s_{t+1}) - d(s_t)\Bigm],
\end{align}
convergence to $d^\pi_P$ when state sequence $\{s_t\}$ is generated under policy $\pi$, kernel $P$ and initial state distribution $\mu$ \cite{borkarBook}.

\begin{lemma}
For all policy $\pi$ and kernel $P$,  the iterative sequence given by  \[ d_{n+1} := \mu + \gamma P^\pi d_n, \quad\forall n\in\mathbb{N},\] converges linearly to $d^{\pi}$, more precisely,
 \begin{align*}
     \lVert d^\pi -d^\pi_{n}\rVert_1 
     &= \gamma^n \lVert d^\pi -d^\pi_{0}\rVert_1, \qquad \forall n\in\mathbb{N}.
 \end{align*}
\label{prop:occupancy_mesaure_bootstrap}
\end{lemma}


\section{Algorithm}
Naturally, our update rule should be
\begin{equation}\label{updateRule}
    \theta \to \theta - \alpha\nabla_\theta f
\end{equation}
where $\alpha$ is the learning rate. In \cite{zhang2020variational}, it has been shown that the update rule \eqref{updateRule} converges to global optima under some conditions for Convex RL. This makes it a very attractive way to solve the RL with general utilities, as it indicates that the update rule \eqref{updateRule} shall converge to local minima in the general setting. Now we explain our core idea to convert the above rule into an algorithm. Observe that if we are sampling states and actions from the policy $\pi$ then $\frac{d\log(\pi(s,a))}{d\theta}Q^{\pi}_{R_{\pi}}(s,a)$ is an unbiased estimate of the gradient $\nabla_\theta f$. And $\frac{d\pi(s,a)}{d\theta}$ is easy to calculate, now we only need the estimate $Q^{\pi}_{R_{\pi}}(s,a)$. For this, we need $R_{\pi}$ that can only be computed by estimating the occupation measure $\mu^{\pi}$. So our algorithm estimates the occupation measure  and Q-values in parallel. Note that estimation of occupation is an estimation of the mean of a random variable, hence it must not be more difficult than the estimation of something more complex as Q-values. So we believe our algorithm will perform as nice as policy gradient methods for Linear RL up to a constant factor.

\begin{algorithm}
\caption{Policy Gradient Algorithm for RL with general utilities}\label{alg:cap}
\begin{algorithmic}
\State Choose appropriate step sizes $\epsilon_t,\eta_t$ for $t\geq 0$. Take $\mu_0(s,a) =1/(|\mathcal{S}||\mathcal{A}|)$, and sample initial state $s_0$ from initial distribution $q$, choose policy parameter $\theta$ randomly. We define $\delta$ as $\delta_{s,a}(s',a') = \mathbf{1}(s'=s,a'=a)$, and  $f' = \nabla_{\mu} f$.\\
\While{ not converged}
    \State  Play action $a_t$ according to the policy $\pi_{\theta}$ 
    \State Update the occupancy measure as 
    \begin{align}
    \mu(s_t) = \mu(s_t) + \eta_t\Bigm[\mu_0(s_t) + \gamma \mu(s_{t+1}) - \mu(s_t)\Bigm],
\end{align}
    \State Get reward vector $R = f'(\mu)$ and sample the next state $s_{t+1}$.
    \State Update Q-value as $$Q(s_t,a_t) = Q(s_t,a_t) + \epsilon_t[R(s_t,a_t) + \gamma \min_{a}Q(s_{t+1},a)-Q(s_t,a_t)].$$
    \State Update policy parameter as 
    $$\theta = \theta - \eta_t \frac{d\log(\pi_{\theta}(s_t,a_t))}{d\theta}Q(s_t,a_t)$$
\EndWhile
\end{algorithmic}
\end{algorithm}

\section{Discussion}
There have been many works that aim to solve special cases of General RL such as Apprentenship Learning, Pure Exploration, etc. Junyu et al. \cite{zhang2020variational} were among the first, who identified these problems as Convex MDP. They showed that Policy Gradient for the Convex MDP would converge to global minima under certain conditions. However, they were unable to obtain policy gradient in Convex RL analogous to policy gradient in Linear RL. But instead, they derived Variational Policy Gradient method using the solution of a stochastic saddle point  problem involving the Fenchel dual of the convex RL objective. In \cite{zahavy2021reward}, the authors used Frenchel duality to convert the convex RL problem into a two-player zero-sum game between the agent ( policy player) and an adversary that produces rewards (cost player) that agent must maximize.\newline\newline
We considered a more general objective and derived Policy Gradient Theorem for RL with general utilities that is as elegant and easily implementable as the Policy Gradient Theorem for Linear RL by \cite{PolicyGradient}. We recover the Policy Gradient Theorem for Linear RL by having Linear RL objective in Policy Gradient Theorem for RL with general utilities.  We believe that it will play the same role in convex RL as Policy Gradient Theorem by \cite{PolicyGradient} played in Linear RL. It gave rise to a very simple algorithm for RL with general utilities. 

\bibliography{main}

\begin{thebibliography}{10}

\bibitem{ApprenticeshipLearning}
Pieter Abbeel and Andrew~Y. Ng.
\newblock Apprenticeship learning via inverse reinforcement learning.
\newblock In {\em Proceedings of the Twenty-First International Conference on
  Machine Learning}, ICML '04, page~1, New York, NY, USA, 2004. Association for
  Computing Machinery.

\bibitem{agarwal2020theory}
Alekh Agarwal, Sham~M. Kakade, Jason~D. Lee, and Gaurav Mahajan.
\newblock On the theory of policy gradient methods: Optimality, approximation,
  and distribution shift, 2020.

\bibitem{borkarBook}
V.S. Borkar.
\newblock {\em Stochastic Approximation: A Dynamical Systems Viewpoint}.
\newblock Cambridge University Press, 2008.

\bibitem{geist2022concave}
Matthieu Geist, Julien Pérolat, Mathieu Laurière, Romuald Elie, Sarah Perrin,
  Olivier Bachem, Rémi Munos, and Olivier Pietquin.
\newblock Concave utility reinforcement learning: the mean-field game
  viewpoint, 2022.

\bibitem{pureExploration}
Elad Hazan, Sham~M. Kakade, Karan Singh, and Abby~Van Soest.
\newblock Provably efficient maximum entropy exploration, 2019.

\bibitem{mnih2015humanlevel}
Volodymyr Mnih, Koray Kavukcuoglu, David Silver, Andrei~A. Rusu, Joel Veness,
  Marc~G. Bellemare, Alex Graves, Martin Riedmiller, Andreas~K. Fidjeland,
  Georg Ostrovski, Stig Petersen, Charles Beattie, Amir Sadik, Ioannis
  Antonoglou, Helen King, Dharshan Kumaran, Daan Wierstra, Shane Legg, and
  Demis Hassabis.
\newblock Human-level control through deep reinforcement learning.
\newblock {\em Nature}, 518(7540):529--533, February 2015.

\bibitem{mutti2022challenging}
Mirco Mutti, Riccardo~De Santi, Piersilvio~De Bartolomeis, and Marcello
  Restelli.
\newblock Challenging common assumptions in convex reinforcement learning,
  2022.

\bibitem{Puterman1994MarkovDP}
Martin~L. Puterman.
\newblock Markov decision processes: Discrete stochastic dynamic programming.
\newblock In {\em Wiley Series in Probability and Statistics}, 1994.

\bibitem{Dsilver2020}
Julian Schrittwieser, Ioannis Antonoglou, Thomas Hubert, Karen Simonyan,
  Laurent Sifre, Simon Schmitt, Arthur Guez, Edward Lockhart, Demis Hassabis,
  Thore Graepel, Timothy Lillicrap, and David Silver.
\newblock Mastering atari, go, chess and shogi by planning with a learned
  model.
\newblock {\em Nature}, 588(7839):604–609, Dec 2020.

\bibitem{Sutton1998}
Richard~S. Sutton and Andrew~G. Barto.
\newblock {\em Reinforcement Learning: An Introduction}.
\newblock The MIT Press, second edition, 2018.

\bibitem{PolicyGradient}
Richard~S Sutton, David McAllester, Satinder Singh, and Yishay Mansour.
\newblock Policy gradient methods for reinforcement learning with function
  approximation.
\newblock In S.~Solla, T.~Leen, and K.~M\"{u}ller, editors, {\em Advances in
  Neural Information Processing Systems}, volume~12. MIT Press, 2000.

\bibitem{zahavy2021reward}
Tom Zahavy, Brendan O'Donoghue, Guillaume Desjardins, and Satinder Singh.
\newblock Reward is enough for convex mdps, 2021.

\bibitem{zhang2020variational}
Junyu Zhang, Alec Koppel, Amrit~Singh Bedi, Csaba Szepesvari, and Mengdi Wang.
\newblock Variational policy gradient method for reinforcement learning with
  general utilities, 2020.

\end{thebibliography}
\bibliographystyle{plain}
\appendix
\section{ Policy Gradient}
\begin{proof}[Proof of Eqn.~\ref{eqn:gradient-occupancy}]
Recall that
\begin{equation}
    \mu^\pi(s,a) = \sum_{t=0}^\infty\gamma^t \sum_{(s_0,a_0,\cdots,s_{t-1},a_{t-1})\in(\mathcal{S}\times\mathcal{A})^t} q(s_0)\prod_{i=0}^{t-1}P(s_{i+1}|s_i,a_i)\prod_{j=0}^t \pi(a_j|s_j)
\end{equation}
where $s_t=s$ and $a_t=a$.
Taking derivative w.r.t. $\theta$ on both sides, we have
\begin{align}
    \nabla_{\theta}\mu^\pi(s,a) & = \sum_{t=0}^\infty\gamma^t \sum_{(s_0,a_0,\cdots,s_{t-1},a_{t-1})\in(\mathcal{S}\times\mathcal{A})^t} q(s_0)\prod_{i=0}^{t-1} P(s_{i+1}|s_i,a_i) \nabla_{\theta} \prod_{j=0}^t \pi(a_j|s_j) \\
     & \text{(using product rule)} \\
    & = \sum_{t=0}^\infty\gamma^t \sum_{(s_0,a_0,\cdots,s_{t-1},a_{t-1})\in(\mathcal{S}\times\mathcal{A})^t} q(s_0)\prod_{i=0}^{t-1} P(s_{i+1}|s_i,a_i)\sum_{k=0}^t \prod_{\substack{j=0\\j\neq k}}^t \pi(a_j|s_j) \nabla_{\theta} \pi(a_k|s_k) \\
    & = \sum_{t=0}^\infty \gamma^t \sum_{k=0}^t \sum_{(s_0,a_0,\cdots,s_{t-1},a_{t-1})\in(\mathcal{S}\times\mathcal{A})^t} q(s_0)\prod_{i=0}^{t-1} P(s_{i+1}|s_i,a_i) \prod_{\substack{j=0\\j\neq k}}^t \pi(a_j|s_j) \nabla_{\theta} \pi(a_k|s_k) \\
    & = \sum_{t=0}^\infty \gamma^t \sum_{k=0}^t \sum_{s_k\in\mathcal{S}} \underbrace{ \sum_{(s_0,a_0,\cdots,s_{k-1},a_{k-1})\in(\mathcal{S}\times\mathcal{A})^k} q(s_0)\prod_{i=0}^{k-1} P(s_{i+1}|s_i,a_i)\pi(a_i|s_i)}_{\alpha^\pi_k(s_k)} \\
    & \qquad \sum_{a_k\in\mathcal{A}}\nabla_{\theta} \pi(a_k|s_k) \underbrace{\sum_{(s_{k+1},a_{k+1},\cdots,s_{t-1},a_{t-1})\in(\mathcal{S}\times\mathcal{A})^k}  \prod_{i={k+1}}^{t} P(s_i|s_{i-1},a_{i-1})\pi(a_i|s_i)}_{ \beta^\pi_{t-k}(s_k,a_k,s_t,a_t)} \\
    & = \sum_{t=0}^\infty  \sum_{k=0}^t \sum_{(s_k, a_k)\in\mathcal{S}\times\mathcal{A}} \gamma^k\alpha^\pi_k(s_k) \gamma^{t-k}\beta^\pi_{t-k}(s_k,a_k,s_t,a_t) \nabla_{\theta} \pi(a_k|s_k)
\end{align}
Note that $\alpha^\pi_k(s)$ is essentially the probability of transiting to state $s$ in exactly $k$ steps from the initial state $s_0\sim q$ under policy $\pi$, so $\mu^\pi(s) = \sum_{k=0}^{\infty}\gamma^k\alpha_k^\pi(s)$. 
Therefore, we can write

\begin{align}
    \nabla_{\theta}\mu^\pi(s,a)
    &= \sum_{(s', a') \in\mathcal{S}\times\mathcal{A}} \sum_{m=0}^\infty \gamma^m \alpha^\pi_m(s') \sum_{n=0}^\infty \gamma^{n}\beta^\pi_n(s',a',s,a) \nabla_{\theta} \pi(a'|s')\\
   & = \sum_{(s',a')\in\mathcal{S}\times\mathcal{A}} \mu^{\pi}(s') \beta^\pi(s',a',s,a) \nabla_{\theta}\pi(a'|s')
\end{align}
which concludes the proof.
\end{proof}

\section{Occupation Measure Bootstraping}
\begin{lemma*}
For all policy $\pi$ and kernel $P$,  the iterative sequence given by  \[ d_{n+1} := \mu + \gamma P^\pi d_n, \quad\forall n\in\mathbb{N},\] converges linearly to $d^{\pi}$.
\end{lemma*}
\begin{proof}
 We first prove,  $d^\pi \in \R^{\St} $ can be written as
 \begin{align*}
     &d^\pi = \mu^T(I-\gamma P^{\pi})^{-1} = \mu^T\sum_{n=0}^{\infty}(P^\pi)^n\\
     \implies & \gamma d^\pi P^\pi =  \Bigm(\mu^T\sum_{n=0}^{\infty}(\gamma P^\pi)^n\Bigm)\gamma P^\pi = d^\pi - I.\\
 \end{align*}
 We conclude that we have
 \[d^\pi = I+\gamma d^\pi P^\pi.\]
 Now, we have 
 \begin{align*}
     \lVert d^\pi -d^\pi_{n+1}\rVert_1 &= \lVert I+\gamma d^\pi P^\pi -\mu - \gamma d_nP^\pi\rVert_1,\qquad \text{(from definition)}\\
     &= \gamma \lVert (d^\pi -  d_n)P^\pi\rVert_1\\
     &\leq \gamma \sum_{s'}\sum_{s}\lvert d^\pi(s) -  d_n(s)\lvert P(s'|s)\\
     &= \gamma \sum_{s}\lvert d^\pi(s) -  d_n(s)\lvert \\
     &= \gamma \lVert d^\pi -d^\pi_{n}\rVert_1.
 \end{align*}
 This proves the claim. Note that convergence in not in $L_\infty$ norm but $L_1$ norm instead.
\end{proof}


\end{document}